\documentclass[a4paper,11pt]{article}
\usepackage[T1]{fontenc}
\usepackage[utf8]{inputenc}
\usepackage{lmodern}
\usepackage{amsmath,amssymb,amsthm,mathtools,bm}
\usepackage{geometry}
\usepackage{enumitem}
\usepackage{booktabs}
\usepackage{xcolor}
\usepackage{hyperref}
\hypersetup{colorlinks=true,linkcolor=blue!55!black,citecolor=blue!55!black,urlcolor=blue!55!black}
\theoremstyle{plain}
\newtheorem{theorem}{Theorem}[section]
\newtheorem{lemma}[theorem]{Lemma}
\newtheorem{proposition}[theorem]{Proposition}
\newtheorem{conjecture}[theorem]{Conjecture}

\theoremstyle{definition}
\newtheorem{definition}[theorem]{Definition}

\theoremstyle{remark}
\newtheorem{remark}[theorem]{Remark}

\DeclareMathOperator{\rank}{rank}
\DeclareMathOperator{\Imop}{Im}
\DeclareMathOperator{\diag}{diag}
\DeclareMathOperator{\adj}{adj}
\DeclareMathOperator{\sgn}{sgn}
\DeclareMathOperator{\logit}{logit}

\title{Moonshine: An Autonomous Mathematical Research Agent Centered on Conjecture Generation}
\author{Xiaoyang Chen, Xiang Jiang}
\date{}
\author{
	Xiaoyang Chen, Xiang Jiang\\
	School of Mathematical Science, Tongji University\\
	Moonshine Technology
}
\begin{document}
\maketitle
\begin{center}
\small\textbf{Project homepage and source code:}\ \url{https://github.com/DeepMathLLM/Moonshine}
\end{center}

\begin{abstract}
Moonshine is an autonomous agent whose central objective is to generate mathematical conjectures. Its core capability is to extract structure from classical problems, distill new concepts, and formulate conjectures of mathematical significance. Rather than treating the solution of a single proposition as its endpoint, Moonshine builds an extensible theoretical framework through conjecture generation, bridge building, and obstacle identification. This article uses Moonshine's exploration of the Jacobian conjecture as an example. It shows how the central logic of whether local nondegeneracy can force global injectivity is transferred to one-hidden-layer affine-ridge sigmoid networks. This leads to the formulation of the \emph{Neural Jacobian Conjecture} (NJC): if such a network has strictly positive Jacobian determinant on the whole space, then it must be globally injective. By invoking GPT-5.5-pro and DeepSeek-V4-pro separately, Moonshine obtained independent complete proofs for the case \(N=n+1\). In addition, with the assistance of ChatGPT through interactive use of its web interface with GPT-5.5-pro, a geometric-topological proof was developed. These results provide preliminary evidence for the plausibility of the conjecture. The general higher-width case \(N\ge n+2\), however, remains unresolved and is left for further investigation. This work illustrates Moonshine's ability to autonomously generate meaningful mathematical problems and make rigorous progress on them.
\end{abstract}

\section{Moonshine as an Autonomous Mathematical Research Agent Centered on Conjecture Generation}

Moonshine is an autonomous mathematical research agent framework. It differs from question-answering systems or numerical computation tools in that its design goal is to \textbf{autonomously generate valuable mathematical conjectures} and to verify or refute them through structured exploration. Moonshine's behavior revolves around the following components:

\begin{itemize}[leftmargin=2em]
\item \textbf{Structural recognition and conjecture distillation.} It identifies core structural features in classical problems or mathematical objects, distills new concepts, and formulates precise and testable conjectures based on them.
\item \textbf{Deep exploration and bridge building.} It connects conjectures with existing theories, explores interactions with other mathematical areas, and derives conditional conclusions.
\item \textbf{Obstacle identification and boundary delineation.} Through proofs and counterexample construction, it clarifies sufficient conditions under which a conjecture holds and identifies barriers that cannot be crossed, thereby determining the true range of applicability.
\item \textbf{Research logs and memory.} It maintains long-term structured logs recording the evolution of conjectures, proof attempts, failed paths, and open subproblems, thereby forming an extensible theoretical framework.
\end{itemize}

Moonshine is organized around a runtime home directory (by default \texttt{\textasciitilde/.moonshine}), which contains configuration files, project folders, session logs, knowledge bases, skills, tools, and MCP server definitions. The agent operates in a research mode, in which it can iterate autonomously, search its historical memory, invoke verification tools, and gradually enrich its understanding of a given conjecture.

\paragraph{Moonshine's exploration of the Jacobian conjecture and the formulation of the Neural Jacobian Conjecture.}
Inspired by the classical complex Jacobian conjecture, Moonshine did not attempt to prove or disprove the original conjecture directly. Instead, it extracted the core logic: whether local nondegeneracy, expressed by a nonzero Jacobian determinant, can force global injectivity. This logic was then transferred to a restricted but structurally transparent function class, namely one-hidden-layer affine-ridge sigmoid networks. By analyzing the special algebraic and geometric structure of this network class, Moonshine distilled a new conjecture, the \textbf{Neural Jacobian Conjecture} (NJC). The conjecture asserts that if the Jacobian determinant of such a network is everywhere positive, then the network must be globally injective. This conjecture is both an analogy with the classical problem and an independently meaningful statement, because it attributes the rigidity from local diffeomorphism to global injectivity to the special affine-ridge structure.

The following sections present Moonshine's exploration of the NJC.

\section{The Neural Jacobian Conjecture (NJC)}

\subsection{Function class and notation}

Let \(\sigma:\mathbb R\to(0,1)\) be the logistic sigmoid function
\[
\sigma(t)=\frac{1}{1+e^{-t}}.
\]
It is strictly increasing and satisfies \(\sigma'(t)=\sigma(t)(1-\sigma(t))>0\) for all \(t\in\mathbb R\).

\begin{definition}[Affine-ridge sigmoid networks]
For \(n,N\ge 1\), a map \(F:\mathbb R^n\to\mathbb R^n\) is called a one-hidden-layer affine-ridge sigmoid network of width \(N\) if
\[
F(x)=W^{(2)}\sigma\bigl(W^{(1)}x+b^{(1)}\bigr)+b^{(2)},
\]
where \(W^{(1)}\in\mathbb R^{N\times n}\), \(W^{(2)}\in\mathbb R^{n\times N}\), \(b^{(1)}\in\mathbb R^N\), \(b^{(2)}\in\mathbb R^n\), and \(\sigma\) acts componentwise. This class is denoted by \(\mathcal N_{n,N}\). When injectivity is unaffected by the notation, we write
\[
F(x)=A\sigma(Bx+c).
\]
\end{definition}

\begin{definition}[Positive-Jacobian subclass]
Define
\[
\mathcal N_{n,N}^+=\{F\in\mathcal N_{n,N}: \det DF(x)>0\text{ for all }x\in\mathbb R^n\}.
\]
\end{definition}

\subsection{The conjecture proposed by Moonshine}

\begin{conjecture}[Neural Jacobian Conjecture]
For every \(F\in\mathcal N_{n,N}^+\), the map \(F\) is globally injective.
\end{conjecture}

The motivation behind the conjecture is that a local diffeomorphism can in general form multiple sheets, whereas the special structure of affine-ridge sigmoid networks - in particular, the interaction between the kernel of the output weight matrix and the image of the hidden layer - may enforce uniqueness. If true, the NJC would provide a rigidity theorem in the neural-network setting, showing that local nondegeneracy implies global injectivity under this special structural constraint, and it would form an interesting parallel with the classical Jacobian conjecture.

\subsection{A geometric reformulation}

Let \(h(x)=\sigma(Bx+c)\), and let
\[
X_1=h(\mathbb R^n)\subset (0,1)^N,\qquad X_2=\ker A\subset\mathbb R^N.
\]
Here \(X_1\) is the hidden-layer submanifold and \(X_2\) is the output kernel. Then \(F\) is injective if and only if
\[
(p+X_2)\cap X_1=\{p\},\qquad \forall p\in X_1.
\]
The positive-Jacobian condition \(\det DF>0\) is equivalent to the transversal intersection of \(X_1\) and \(p+X_2\) at every point, with local intersection index \(+1\). Hence the NJC can be reformulated as follows: do transversality and positive local index force every affine fiber to intersect \(X_1\) uniquely?

\section{Partial Verification of the NJC in Low Width}

Moonshine does not claim to have proved the NJC in full. It first analyzes the most accessible cases, namely when the hidden-layer width \(N\) equals the input dimension \(n\), or is larger by one. These cases provide initial evidence for the conjecture.

\subsection{The case \texorpdfstring{\(N=n\)}{N=n}}

\begin{proposition}
If \(N=n\) and \(F\in\mathcal N_{n,n}^+\), then \(F\) is injective.
\end{proposition}

\begin{proof}
When \(N=n\), the matrices \(A,B\in\mathbb R^{n\times n}\). Define
\[
D(x)=\diag\bigl(\sigma'((Bx+c)_1),\ldots,\sigma'((Bx+c)_{n})\bigr).
\]

Since
\[
\det DF(x)=\det A\cdot \det D(x)\cdot \det B>0
\]
for all \(x\), and \(\det D(x)>0\), both \(\det A\) and \(\det B\) are nonzero. Thus \(A\) and \(B\) are invertible. The map \(F(x)=A\sigma(Bx+c)\) is a composition of three injective maps: the invertible affine map \(x\mapsto Bx+c\), the componentwise strictly increasing map \(z\mapsto \sigma(z)\), and the invertible linear map \(w\mapsto Aw\). Hence \(F\) is injective.
\end{proof}

\subsection{The case \texorpdfstring{\(N=n+1\)}{N=n+1}}

This is the smallest width in which \(A\) has a nontrivial kernel. The main result of this section is the following.

\begin{theorem}\label{thm:Nnplus1}
Let \(F:\mathbb R^n\to\mathbb R^n\) be a one-hidden-layer affine-ridge sigmoid network with \(N=n+1\) hidden units. If
\[
\det DF(x)>0,
\qquad \forall x\in\mathbb R^n,
\]
then \(F\) is injective.
\end{theorem}

The first proof below was obtained by Moonshine using GPT-5.5-pro.

\subsubsection{First algebraic proof: an injectivity lemma on convex sets}

As a supplement to Theorem~\ref{thm:Nnplus1}, we give a proof whose core is the following injectivity lemma for a linear projection applied to the graph of a scalar function.

\begin{lemma}[Injectivity lemma on convex sets]\label{lem:convex-inj}
Let \(\Omega\subset\mathbb R^n\) be a nonempty open convex set, let \(h\in C^1(\Omega)\), and let \(L:\mathbb R^{n+1}\to\mathbb R^n\) be linear. Define
\[
T(y)=L(y,h(y)).
\]
If \(\det DT(y)\ne0\) for all \(y\in\Omega\), then \(T\) is injective.
\end{lemma}

\begin{proof}
Since \(\det DT(y)\ne0\), we have \(\rank DT(y)=n\). Hence \(\rank L=n\), because the image of \(DT(y)\) is contained in the image of \(L\). Thus \(\dim\ker L=1\). Let \(k=(p,r)\in\mathbb R^n\times\mathbb R\) be a nonzero vector spanning \(\ker L\). There are two cases.

\emph{Case 1: \(r\ne0\).} Rescale \(k\) so that \(r=-1\), hence \(k=(p,-1)\). Define \(Q(y,z)=y+pz\). Then \(\ker Q=\mathbb R(p,-1)=\ker L\). Since both \(L\) and \(Q\) are full-rank maps from \(\mathbb R^{n+1}\) to \(\mathbb R^n\), there exists \(A_0\in GL_n(\mathbb R)\) such that \(A_0L=Q\). Acting on the target by \(A_0\) does not affect injectivity, so we may assume
\[
T(y)=y+p h(y).
\]
Suppose \(T(y_1)=T(y_2)\). Then
\[
y_2-y_1=p\bigl(h(y_1)-h(y_2)\bigr).
\]
Thus \(y_2-y_1\) is parallel to \(p\). If \(p=0\), then \(T(y)=y\), and there is nothing to prove. Assume \(p\ne0\). Then \(y_2=y_1+sp\) for some \(s\ne0\). By convexity of \(\Omega\), the segment \(\{y_1+tp:0\le t\le s\}\) is contained in \(\Omega\). The equality above gives
\[
s+h(y_1+sp)-h(y_1)=0.
\]
Set \(\psi(t)=t+h(y_1+tp)\). Then \(\psi(s)=\psi(0)\). Moreover,
\[
\psi'(t)=1+\nabla h(y_1+tp)\cdot p.
\]
By the matrix determinant lemma,
\[
\det DT(y_1+tp)=\det\bigl(I+p(\nabla h)^T\bigr)=1+\nabla h(y_1+tp)\cdot p.
\]
Hence \(\psi'(t)=\det DT(y_1+tp)\). By assumption this never vanishes; since it is continuous, its sign is constant, and \(\psi\) is strictly monotone. This contradicts \(\psi(s)=\psi(0)\) unless \(s=0\). Therefore \(y_1=y_2\).

\emph{Case 2: \(r=0\).} Then \(\ker L=\mathbb R(p,0)\) with \(p\ne0\). Choose a linear isomorphism
\[
P:\mathbb R^{n-1}\times\mathbb R\to\mathbb R^n
\]
such that \(P(0,1)=p\). Let
\[
\widetilde\Omega=P^{-1}(\Omega),\qquad \widetilde h(u,s)=h(P(u,s)).
\]
In coordinates \((u,s,z)\in\mathbb R^{n-1}\times\mathbb R\times\mathbb R\), define
\[
\widetilde L(u,s,z)=L(P(u,s),z).
\]
Then \(\ker\widetilde L=\mathbb R(0,1,0)\). Consider
\[
H=\mathbb R^{n-1}\times\{0\}\times\mathbb R\cong\mathbb R^n.
\]
The restriction \(\widetilde L|_H:H\to\mathbb R^n\) is a linear isomorphism. Let
\[
J:H\to\mathbb R^n,
\qquad J(u,0,z)=(u,z),
\]
and define
\[
M=J\circ(\widetilde L|_H)^{-1}\in GL_n(\mathbb R).
\]
Since \((0,s,0)\in\ker\widetilde L\), we have
\[
M\widetilde L(u,s,z)=J(u,0,z)=(u,z).
\]
After the source coordinate change \(y=P(u,s)\) and the target linear change \(M\), the map becomes
\[
\widetilde T(u,s)=M T(P(u,s))=M\widetilde L(u,s,\widetilde h(u,s))=(u,\widetilde h(u,s)).
\]
These transformations are invertible, so they do not affect injectivity; the Jacobian determinant is only multiplied by a nonzero constant.

The Jacobian of \(\widetilde T\) is
\[
D\widetilde T(u,s)=
\begin{pmatrix}
I_{n-1} & 0\\
\partial_u\widetilde h(u,s) & \partial_s\widetilde h(u,s)
\end{pmatrix},
\]
and hence \(\det D\widetilde T(u,s)=\partial_s\widetilde h(u,s)\). Therefore
\[
\partial_s\widetilde h(u,s)\ne0,
\qquad (u,s)\in\widetilde\Omega.
\]
For fixed \(u\), the fiber
\[
I_u=\{s\in\mathbb R:(u,s)\in\widetilde\Omega\}
\]
is an interval, since \(\widetilde\Omega\) is convex. The function \(s\mapsto\widetilde h(u,s)\) has a continuous derivative which never vanishes on \(I_u\), hence it is strictly monotone. If \(\widetilde T(u_1,s_1)=\widetilde T(u_2,s_2)\), the first \(n-1\) components give \(u_1=u_2\), and the last component gives \(\widetilde h(u_1,s_1)=\widetilde h(u_1,s_2)\). Strict monotonicity then yields \(s_1=s_2\). Hence \(\widetilde T\), and therefore \(T\), is injective.
\end{proof}

\begin{proof}[Proof of Theorem~\ref{thm:Nnplus1}, first algebraic form]
Write \(F(x)=A\sigma(Bx+c)\), where \(A\in\mathbb R^{n\times(n+1)}\), \(B\in\mathbb R^{(n+1)\times n}\), and \(c\in\mathbb R^{n+1}\). Since \(\rank B=n\), some \(n\) rows of \(B\) are linearly independent. After reordering the hidden units, assume that the first \(n\) rows are independent. By an invertible affine change of input variables, the first \(n\) preactivations can be normalized to \(x_1,\ldots,x_n\). Thus the network can be written as
\[
F(x)=C\sigma(x)+w\sigma(a^T x+\beta),
\]
where \(C\in\mathbb R^{n\times n}\), \(w\in\mathbb R^n\), \(a\in\mathbb R^n\), \(\beta\in\mathbb R\), and \(b\in\mathbb R^n\). Equivalently,
\[
B=\begin{pmatrix} I_n\\ a^T\end{pmatrix},
\qquad
A=\begin{pmatrix} C & w\end{pmatrix}.
\]
Let \(\Phi:\mathbb R^n\to(0,1)^n\) be \(\Phi(x)=\sigma(x)\), and set
\[
h(y)=\sigma\bigl(a^T\Phi^{-1}(y)+\beta\bigr),
\qquad y\in(0,1)^n.
\]
Then
\[
F(x)=Cy+w h(y)=T(y),
\qquad y=\Phi(x).
\]
Since \(\det DF(x)>0\) and \(\Phi\) is a diffeomorphism, \(\det DT(y)\ne0\) for all \(y\in(0,1)^n\). The domain \((0,1)^n\) is open and convex, so Lemma~\ref{lem:convex-inj} implies that \(T\) is injective. Hence \(F\) is injective.
\end{proof}

This proof highlights the fundamental role of convexity and a one-dimensional kernel in the low-width case of the NJC.

\subsubsection{Second algebraic proof: one-dimensional monotonicity along the kernel direction}

The following second proof was obtained by Moonshine using DeepSeek-V4-pro.

\begin{proof}
Again write \(F(x)=A\sigma(Bx+c)\), where \(A\in\mathbb R^{n\times(n+1)}\), \(B\in\mathbb R^{(n+1)\times n}\), and \(c\in\mathbb R^{n+1}\). As above, after normalization the network can be written as
\[
F(x)=C\sigma(x)+w\sigma(a^Tx+\beta).
\]
Its Jacobian is
\[
DF(x)=C S(x)+s_{n+1}(x)w a^T,
\]
where
\[
S(x)=\diag(\sigma'(x_1),\ldots,\sigma'(x_n)),
\qquad
s_{n+1}(x)=\sigma'(a^Tx+\beta).
\]
All entries of \(S(x)\) and \(s_{n+1}(x)\) are positive.

Suppose, for contradiction, that \(F\) is not injective. Then there exist \(p\ne q\) such that \(F(p)=F(q)\). Let
\[
u=\sigma(p)\in(0,1)^n,
\qquad
v=\sigma(q)\in(0,1)^n.
\]
Define
\[
G(u)=\sigma\bigl(a^T\sigma^{-1}(u)+\beta\bigr),
\]
where \(\sigma^{-1}(s)=\log(s/(1-s))\) acts componentwise. Then
\[
F(x)=C u+wG(u),
\qquad u=\sigma(x).
\]
Since \(A=[C,w]\) has rank \(n\), its kernel is one-dimensional. Choose
\[
k=\begin{pmatrix}\widehat k\\ k_{n+1}\end{pmatrix}\in\ker A\setminus\{0\}.
\]
Then
\[
C\widehat k+k_{n+1}w=0. \tag{1}
\]
The equality \(F(p)=F(q)\) implies that the difference of the hidden-layer outputs lies in \(\ker A\). Thus for some \(\lambda\ne0\),
\[
\begin{pmatrix}v\\G(v)\end{pmatrix}-\begin{pmatrix}u\\G(u)\end{pmatrix}=\lambda k.
\]
In particular,
\[
v=u+\lambda\widehat k,
\qquad
G(v)-G(u)=\lambda k_{n+1}. \tag{2}
\]
Define
\[
f(t)=G(u+t\widehat k)-G(u)-t k_{n+1}
\]
for those \(t\) for which \(u+t\widehat k\in(0,1)^n\). By (2), \(f(0)=f(\lambda)=0\), with \(\lambda\ne0\). Let
\[
x_t=\sigma^{-1}(u+t\widehat k).
\]
Convexity of \((0,1)^n\) ensures that \(x_t\) is well-defined for \(t\) between \(0\) and \(\lambda\). Differentiating gives
\[
f'(t)=\nabla G(u+t\widehat k)\cdot\widehat k-k_{n+1}.
\]
A direct computation yields
\[
\nabla G(u+t\widehat k)=s_{n+1}(x_t)a^T S(x_t)^{-1},
\]
and hence
\[
f'(t)=s_{n+1}(x_t)a^T S(x_t)^{-1}\widehat k-k_{n+1}. \tag{3}
\]
We relate this expression to \(\det DF(x_t)\).

\emph{Case 1: \(k_{n+1}\ne0\).} Rescale \(k\) so that \(k_{n+1}=-1\). Then (1) gives \(w=C\widehat k\). Therefore
\[
DF(x_t)=C\bigl(S(x_t)+s_{n+1}(x_t)\widehat k a^T\bigr).
\]
Since \(A=[C,w]\) has rank \(n\), and \(w=C\widehat k\), the matrix \(C\) is invertible. The matrix determinant lemma gives
\[
\det DF(x_t)=\det C\det S(x_t)\bigl(1+s_{n+1}(x_t)a^TS(x_t)^{-1}\widehat k\bigr).
\]
By (3),
\[
f'(t)=1+s_{n+1}(x_t)a^TS(x_t)^{-1}\widehat k
=\frac{\det DF(x_t)}{\det C\det S(x_t)}.
\]
Since \(\det DF(x_t)>0\), \(\det S(x_t)>0\), and \(\det C\ne0\) is constant, \(f'\) has a fixed nonzero sign on the interval between \(0\) and \(\lambda\). Hence \(f\) is strictly monotone there.

\emph{Case 2: \(k_{n+1}=0\).} Then \(C\widehat k=0\), with \(\widehat k\ne0\). Thus \(\rank C\le n-1\). Since \([C,w]\) has rank \(n\), we have \(\rank C=n-1\), and \(w\notin\Imop C\). Choose \(v_0\ne0\) spanning the left kernel \(\ker C^T\). Then \(v_0^Tw\ne0\). The adjugate matrix has rank one and can be written as
\[
\adj(C)=\gamma\widehat k v_0^T,
\qquad \gamma\ne0. \tag{4}
\]
Writing
\[
DF(x_t)=\bigl(C+s_{n+1}(x_t)w a^T S(x_t)^{-1}\bigr)S(x_t),
\]
and using the rank-one perturbation formula for a rank \(n-1\) matrix,
\[
\det(C+uv^T)=v^T\adj(C)u,
\]
we obtain
\[
\det DF(x_t)=\det S(x_t)\,s_{n+1}(x_t)a^TS(x_t)^{-1}\widehat k\,\gamma(v_0^Tw).
\]
Since \(k_{n+1}=0\), (3) becomes
\[
f'(t)=s_{n+1}(x_t)a^TS(x_t)^{-1}\widehat k.
\]
Therefore
\[
f'(t)=\frac{\det DF(x_t)}{\det S(x_t)\gamma(v_0^Tw)}.
\]
Again \(\det DF(x_t)>0\), \(\det S(x_t)>0\), and \(\gamma(v_0^Tw)\ne0\), so \(f'\) has a fixed nonzero sign. Thus \(f\) is strictly monotone.

In both cases, \(f\) is strictly monotone on the closed interval with endpoints \(0\) and \(\lambda\), but \(f(0)=f(\lambda)=0\) and \(\lambda\ne0\), a contradiction. Hence no distinct \(p,q\) with \(F(p)=F(q)\) exist, and \(F\) is injective.
\end{proof}

\subsubsection{Geometric-topological proof: the hidden-layer submanifold and one-dimensional fibers}

The following geometric-topological proof was developed with the assistance of ChatGPT, using GPT-5.5-pro through interactive use of its web interface.

Use the notation
\[
F(x)=A\sigma(Bx+c),
\]
where
\[
A\in\mathbb R^{n\times(n+1)},
\qquad
B\in\mathbb R^{(n+1)\times n}.
\]
Assume
\[
\rank A=\rank B=n,
\qquad
\det DF(x)>0\quad\forall x\in\mathbb R^n.
\]
Let
\[
h(x)=\sigma(Bx+c),
\]
and define
\[
X_1=h(\mathbb R^n)=\sigma(c+\Imop B)\subset(0,1)^{n+1},
\qquad
X_2=\ker A.
\]
Since \(B\) has full column rank, \(c+\Imop B\) is an affine \(n\)-plane in \(\mathbb R^{n+1}\). Since the componentwise sigmoid is a diffeomorphism from \(\mathbb R^{n+1}\) to \((0,1)^{n+1}\), \(X_1\) is a smoothly embedded \(n\)-dimensional submanifold of \((0,1)^{n+1}\). Moreover \(\dim X_2=1\).

By the hidden-layer intersection formulation, \(F\) is injective if and only if
\[
(p+X_2)\cap X_1=\{p\},
\qquad \forall p\in X_1.
\]
It is therefore enough to prove that every affine line with direction \(X_2\) meets \(X_1\) at most once.

Fix \(p\in X_1\), and choose \(k\in X_2\setminus\{0\}\) with \(X_2=\mathbb R k\). Define
\[
I_p=\{t\in\mathbb R:p+tk\in(0,1)^{n+1}\}.
\]
This is an open interval containing \(0\).

We describe \(X_1\) in logit coordinates. Let
\[
\logit(z)=\left(\log\frac{z_1}{1-z_1},\ldots,\log\frac{z_{n+1}}{1-z_{n+1}}\right).
\]
Then
\[
z\in X_1
\iff
\logit(z)-c\in\Imop B.
\]
Choose \(0\ne\lambda\in(\Imop B)^\perp\), and define
\[
\Phi:(0,1)^{n+1}\to\mathbb R,
\qquad
\Phi(z)=\lambda\cdot(\logit(z)-c).
\]
Then
\[
X_1=\{z\in(0,1)^{n+1}:\Phi(z)=0\}.
\]
Restrict \(\Phi\) to the affine line \(p+X_2\), and define
\[
g:I_p\to\mathbb R,
\qquad
g(t)=\Phi(p+tk).
\]
Then
\[
p+tk\in X_1\iff g(t)=0.
\]
Thus intersections of \(p+X_2\) with \(X_1\) correspond exactly to zeros of \(g\), and \(g(0)=0\).

We first prove that if \(t_0\in I_p\) is a zero of \(g\), then
\[
g'(t_0)\ne0.
\]
This assertion concerns the derivative at zeros only; it does not say that \(g'\) is nonzero on all of \(I_p\).

Let \(g(t_0)=0\), and set \(z_0=p+t_0k\). Then \(z_0\in X_1\), so \(z_0=h(x_0)=\sigma(Bx_0+c)\) for some \(x_0\). Define
\[
D(x_0)=\diag\bigl(\sigma'((Bx_0+c)_1),\ldots,\sigma'((Bx_0+c)_{n+1})\bigr).
\]
Then
\[
T_{z_0}X_1=D(x_0)\Imop B.
\]

\begin{lemma}\label{lem:not-tangent}
At every zero \(t_0\) of \(g\), one has
\[
k\notin T_{z_0}X_1.
\]
\end{lemma}

\begin{proof}
If \(k\in T_{z_0}X_1\), then \(k=D(x_0)Bv\) for some \(v\in\mathbb R^n\). Since \(k\in X_2=\ker A\),
\[
0=Ak=AD(x_0)Bv.
\]
But \(AD(x_0)B=DF(x_0)\), and \(\det DF(x_0)>0\), so \(AD(x_0)B\) is invertible. Hence \(v=0\), and therefore \(k=0\), a contradiction.
\end{proof}

Since \(X_1\) is the zero level set of \(\Phi\), and on \(X_1\)
\[
T_{z_0}X_1=\ker d\Phi_{z_0},
\]
Lemma~\ref{lem:not-tangent} implies
\[
d\Phi_{z_0}(k)\ne0.
\]
By the chain rule, because \(g(t)=\Phi(p+tk)\),
\[
g'(t_0)=d\Phi_{z_0}(k).
\]
Thus every zero of \(g\) is nondegenerate, in the sense that its derivative is nonzero.

We next show that the signs of \(g'\) at all zeros are the same. For this purpose we use a separate linear-algebra lemma.

\begin{lemma}\label{lem:det-comparison}
Let
\[
A:\mathbb R^{n+1}\to\mathbb R^n,
\qquad
B:\mathbb R^n\to\mathbb R^{n+1}
\]
be linear maps with \(\rank A=\rank B=n\). Suppose
\[
\ker A=\mathbb R k,
\qquad
0\ne\lambda\in(\Imop B)^\perp.
\]
Then there exists a nonzero constant \(C=C(A,B,k,\lambda)\) such that for every \(L\in GL_{n+1}(\mathbb R)\),
\[
\det(ALB)=C\det L\,\lambda^TL^{-1}k.
\]
In particular, if \(\det L>0\), then there is a fixed sign \(\varepsilon\in\{\pm1\}\), depending only on \(A,B,k,\lambda\), such that
\[
\sgn\det(ALB)=\varepsilon\,\sgn(\lambda^TL^{-1}k).
\]
\end{lemma}

\begin{proof}
First, for any \(M=[m_1,\ldots,m_n]\in\mathbb R^{(n+1)\times n}\), since \(A\) has rank \(n\) and \(\ker A=\mathbb Rk\), there is a nonzero constant \(c_A\), depending only on \(A\) and \(k\), such that
\[
\det(AM)=c_A\det[M,k].
\]
Indeed, choose a basis \(u_1,\ldots,u_n\) of a complement to \(\mathbb Rk\), and set \(U=[u_1,\ldots,u_n]\). Then \([U,k]\) is invertible and \(AU\) is invertible. Every \(M\) can be written uniquely as \(M=UY+k\alpha^T\). Then
\[
\det(AM)=\det(AU)\det Y,
\qquad
\det[M,k]=\det[U,k]\det Y,
\]
which gives the desired proportionality.

Second, since \(\rank B=n\), the image \(\Imop B\) is an \(n\)-dimensional hyperplane in \(\mathbb R^{n+1}\). The map \(q\mapsto\det[B,q]\) is a linear functional vanishing on \(\Imop B\), as is \(q\mapsto\lambda^Tq\). Since the space of such functionals is one-dimensional, there exists \(c_B\ne0\) such that
\[
\det[B,q]=c_B\lambda^Tq,
\qquad \forall q\in\mathbb R^{n+1}.
\]
Now take \(M=LB\). Then
\[
\det(ALB)=c_A\det[LB,k].
\]
Since \([LB,k]=L[B,L^{-1}k]\),
\[
\det[LB,k]=\det L\det[B,L^{-1}k]
=\det L\,c_B\lambda^TL^{-1}k.
\]
Thus
\[
\det(ALB)=c_Ac_B\det L\,\lambda^TL^{-1}k.
\]
Set \(C=c_Ac_B\ne0\). The sign statement follows immediately.
\end{proof}

Return to the geometric proof. At a zero \(t_0\), write again \(z_0=p+t_0k=h(x_0)\). Since \(z_0=\sigma(Bx_0+c)\),
\[
D(x_0)=\diag\bigl(z_{0,1}(1-z_{0,1}),\ldots,z_{0,n+1}(1-z_{0,n+1})\bigr),
\]
which is a positive diagonal matrix. Moreover,
\[
d\Phi_{z_0}(w)=\lambda^TD(x_0)^{-1}w,
\qquad
\text{so}
\qquad
d\Phi_{z_0}(k)=\lambda^TD(x_0)^{-1}k.
\]
Applying Lemma~\ref{lem:det-comparison} with \(L=D(x_0)\), we get
\[
\det(AD(x_0)B)=C\det D(x_0)\,d\Phi_{z_0}(k).
\]
Since \(\det D(x_0)>0\), there exists a fixed sign \(\varepsilon\), independent of the intersection point, such that
\[
\sgn\det(AD(x_0)B)=\varepsilon\,\sgn d\Phi_{z_0}(k).
\]
But \(\det(AD(x_0)B)=\det DF(x_0)>0\), so all zeros \(t_0\) have the same sign of \(d\Phi_{z_0}(k)\), and hence the same sign of \(g'(t_0)\).

Finally, we use the following elementary one-dimensional fact: if \(g:I\to\mathbb R\) is \(C^1\), and \(t_1<t_2\) are two adjacent zeros with \(g'(t_1)\ne0\) and \(g'(t_2)\ne0\), then
\[
\sgn g'(t_1)=-\sgn g'(t_2).
\]
Indeed, if \(g'(t_1)>0\), then \(g\) is positive just to the right of \(t_1\). Since there are no zeros in \((t_1,t_2)\), it remains positive on that interval. Since \(g(t_2)=0\) and \(g'(t_2)\ne0\), it follows that \(g'(t_2)<0\). The other case is analogous.

If there were another intersection, then for some \(\lambda_0\in I_p\setminus\{0\}\), \(g(\lambda_0)=0\). On the compact interval with endpoints \(0\) and \(\lambda_0\), the zeros are isolated and hence finite. Thus two adjacent zeros can be chosen. The one-dimensional fact forces the signs of \(g'\) at these two zeros to be opposite, whereas the argument above shows that all such signs are equal. This contradiction proves
\[
g^{-1}(0)=\{0\}.
\]
Equivalently,
\[
(p+X_2)\cap X_1=\{p\}.
\]
Since \(p\in X_1\) was arbitrary, \(F\) is globally injective.

This proof reveals the essence of the case \(N=n+1\): the one-dimensional output kernel reduces the intersection problem to a one-variable zero problem; nondegenerate zeros in one dimension have alternating derivative signs, while the positive Jacobian condition forces all local signs to agree. This mechanism no longer holds automatically for higher-dimensional fibers, which explains the difficulty of the case \(N\ge n+2\).

\subsection{The higher-width case \texorpdfstring{\(N\ge n+2\)}{N >= n+2} remains open}

For \(N\ge n+2\), the dimension of the kernel \(X_2\) is at least two. Local positive index no longer excludes multiple intersections through one-dimensional sign alternation. Higher-dimensional maps can have multiple regular zeros, all with local index \(+1\), while still failing to be injective. Thus the Neural Jacobian Conjecture remains open in this higher-width regime.

\section{Conclusion and Outlook}

By reflecting on the classical Jacobian conjecture, Moonshine extracted the core principle ``local nondegeneracy implies global injectivity'' and transferred it to one-hidden-layer affine-ridge sigmoid networks, thereby proposing the Neural Jacobian Conjecture. If fully true, the conjecture would reveal an intrinsic rigidity of a special class of neural networks. Even if it is eventually disproved, its exploration helps clarify the boundary between local diffeomorphism and global injectivity.

Moonshine has proved the NJC in the lowest nontrivial widths \(N=n\) and \(N=n+1\), providing initial evidence for its plausibility. For the general higher-width case \(N\ge n+2\), the conjecture is neither proved nor disproved and remains an active open problem. This exemplifies Moonshine's working mode as a conjecture-generating mathematical agent: it formulates precise conjectures, establishes rigorous partial results, and identifies the unresolved boundary that guides subsequent research.

The complete source code, research logs, and intermediate verification records are available in the project repository: \url{https://github.com/DeepMathLLM/Moonshine}.

\appendix
\section{Supplementary Remarks}

\begin{remark}[The Jacobian determinant cannot be a nonzero constant]
For \(N>n\), the Jacobian determinant \(\det DF(x)\) of a network \(F\in\mathcal N_{n,N}\) can never be a nonzero constant. By the Cauchy-Binet formula,
\[
\det(AD(x)B)=\sum_{|I|=n}\det(A_I)\det(B_I)\prod_{i\in I}\sigma'(b_i\cdot x+c_i),
\]
where \(I\) ranges over all \(n\)-element subsets of \(\{1,\ldots,N\}\), \(A_I\) is the \(n\times n\) submatrix of \(A\) with columns in \(I\), \(B_I\) is the \(n\times n\) submatrix of \(B\) with rows in \(I\), and the product consists of positive diagonal entries.

Choose a direction \(u\in\mathbb R^n\) such that \(b_i\cdot u\ne0\) for all nonzero rows \(b_i\). Along the ray \(x=tu\), as \(t\to+\infty\), each preactivation tends to \(\pm\infty\), and hence \(\sigma'(b_i\cdot(tu)+c_i)\to0\). Thus every summand tends to zero, and \(\det DF(tu)\to0\). If \(\det DF\) were a constant \(c\), this limit would force \(c=0\). Hence the determinant cannot be a nonzero constant.
\end{remark}

\begin{remark}[Necessity of full-rank conditions]
If \(\rank A<n\) or \(\rank B<n\), then \(AD(x)B\) has rank at most \(\min(\rank A,\rank B)<n\), and therefore \(\det DF(x)=0\) for all \(x\). Thus any \(F\in\mathcal N_{n,N}^+\) automatically satisfies \(\rank A=\rank B=n\), which in particular requires \(N\ge n\).
\end{remark}

\begin{remark}[Other activation functions]
The proofs use only the facts that \(\sigma'(t)>0\) and that \(\sigma\) has a smooth inverse on its image. Therefore the same result holds for any strictly increasing \(C^1\) activation function mapping \(\mathbb R\) diffeomorphically onto an open interval, such as \((0,1)\). The logistic sigmoid is merely a convenient example.
\end{remark}

\begin{remark}[The complex case is not a direct analogue of the real case]
The NJC considered in this paper is a real-domain statement. If sigmoid networks are complexified directly, the analogous Jacobian-type conclusion generally fails.

Consider the complex logistic sigmoid
\[
\sigma(z)=\frac{1}{1+e^{-z}}.
\]
It is meromorphic, with poles at
\[
z=(2k+1)\pi i,
\qquad k\in\mathbb Z.
\]
For a complex shallow sigmoid network
\[
F(z)=W\sigma(Az+b)+c,
\qquad z\in\mathbb C^n,
\]
where \(A\in\mathbb C^{N\times n}\), \(W\in\mathbb C^{n\times N}\), and \(b\in\mathbb C^N\), let the \(j\)-th preactivation be
\[
\ell_j(z)=a_j^Tz+b_j,
\qquad j=1,\ldots,N.
\]
The natural holomorphic domain is
\[
\Omega=\mathbb C^n\setminus\bigcup_{j=1}^N\bigcup_{k\in\mathbb Z}\{z\in\mathbb C^n:\ell_j(z)=(2k+1)\pi i\}.
\]
On this domain the complex sigmoid satisfies
\[
\sigma(\zeta+2\pi i)=\sigma(\zeta).
\]
Thus one may define the period lattice of the input matrix \(A\) by
\[
L(A)=\{v\in\mathbb C^n:Av\in(2\pi i\mathbb Z)^N\}.
\]
If \(0\ne v\in L(A)\), then \(\sigma(A(z+v)+b)=\sigma(Az+b)\), and hence \(F(z+v)=F(z)\). Such a period vector gives non-injectivity.

In the square case \(N=n\), if \(A,W\in GL_n(\mathbb C)\), then
\[
DF(z)=W\diag(\sigma'(Az+b))A.
\]
On \(\Omega\), \(\sigma'\) has no zeros, so \(\det DF(z)\ne0\). However, for any nonzero \(m\in\mathbb Z^n\),
\[
v=A^{-1}(2\pi i m)
\]
is a nonzero period vector, and therefore \(F(z+v)=F(z)\). Thus the implication
\[
\det DF(z)\ne0\text{ on }\Omega
\quad\Longrightarrow\quad
F\text{ is injective on }\Omega
\]
is false in the complex setting.

For \(N>n\), a second mechanism appears: cancellation by the output kernel. For example, with \(n=2\) and \(N=3\), define
\[
F(z_1,z_2)=
\begin{pmatrix}
\sigma(z_1+\sqrt2 z_2)-\sigma(z_1+z_2)\\
\sigma(\sqrt3 z_1+z_2)-\sigma(z_1+z_2)
\end{pmatrix}.
\]
This corresponds to
\[
A=
\begin{pmatrix}
1&\sqrt2\\
\sqrt3&1\\
1&1
\end{pmatrix},
\qquad
W=
\begin{pmatrix}
1&0&-1\\
0&1&-1
\end{pmatrix}.
\]
Take
\[
z=\left(\frac{2\pi i}{\sqrt3-1},\frac{2\pi i}{\sqrt2-1}\right).
\]
Then
\[
z_1+\sqrt2z_2=(z_1+z_2)+2\pi i,
\qquad
\sqrt3z_1+z_2=(z_1+z_2)+2\pi i.
\]
By periodicity, \(F(z)=0=F(0)\), with \(z\ne0\). Hence this network is not injective.
\end{remark}

\end{document}